\documentclass{article}
\usepackage{kr}

\usepackage{times}
\usepackage{soul}
\usepackage{url}
\usepackage[hidelinks]{hyperref}
\usepackage[utf8]{inputenc}
\usepackage[small]{caption}
\usepackage{graphicx}
\usepackage{amsmath}
\usepackage{amssymb}
\usepackage{amsthm}
\usepackage{booktabs}
\usepackage{algorithm}
\usepackage{algpseudocode}
\usepackage{subcaption}

\newcommand{\bmark}{\mathrel{\mid\!\sim}}

\newtheorem{definition}{Definition}
\newtheorem{proposition}{Proposition}
\newtheorem{remark}{Remark}

\title{Elenchus: Generating Knowledge Bases from Prover-Skeptic Dialogues}

\author{%
    Bradley P. Allen
    \affiliations
    University of Amsterdam
    \emails
    b.p.allen@uva.nl    % email
}
\begin{document}

\maketitle

\begin{abstract}
We present Elenchus, a dialogue system for knowledge base construction grounded in inferentialist semantics, where knowledge engineering is re-conceived as explicitation rather than extraction from expert testimony or textual content. A human expert develops a bilateral position (commitments and denials) about a topic through prover-skeptic dialogue with a large language model (LLM) opponent. The LLM proposes tensions (claims that parts of the position are jointly incoherent) which the expert resolves by retraction, refinement, or contestation. The LLM thus serves as a defeasible derivability oracle whose unreliability is structurally contained by the expert's authority. Our main technical contribution is a mapping from Elenchus dialectical states to material bases in Hlobil and Brandom's NonMonotonic MultiSuccedent (NMMS) logic, satisfying Containment and enabling the elaboration of logical vocabulary that makes explicit the inferential relationships negotiated in the dialectic. We demonstrate the approach on the W3C PROV-O provenance ontology, where a single dialogue session elicits and structures design tensions that a domain expert can articulate, corresponding to decisions documented in a retrospective analysis of the ontology's design. Using pyNMMS, an automated NMMS reasoner, we verify that the structural properties of the resulting material base---nontransitivity, nonmonotonicity, and independence---correspond to specific PROV design rationales, demonstrating end-to-end integration from dialogue through formal reasoning.
\end{abstract}

\section{Introduction}
\label{sect:intro}

Knowledge engineering—the collection of activities for formalizing knowledge for use in information systems~\cite{studer1998knowledge}—has faced the knowledge acquisition problem—the difficulty of the formalization~\cite{dutilh2012formal} of expert knowledge expressed in natural language—since its beginning~\cite{feigenbaum1977art}. Despite fifty years of effort across expert systems, ontology, and knowledge graph engineering, the bottleneck persists.

We argue this persistence stems from a mistaken conception of what knowledge engineering is. Traditional approaches treat expert knowledge as determinate internal content awaiting extraction—as if experts have formal structures in their heads that need only be transcribed~\cite{forsythe1993engineering}. But knowledge is not determinate prior to articulation; it is constituted through practices of expression and negotiation.

Drawing on inferentialist semantics \cite{brandom1994making,hlobil2025reasons}, we propose an alternative: knowledge engineering as \emph{explicitation}, where the task is not to extract pre-formed content but to make explicit, through dialogue, the inferential relationships implicit in expert practice. On this view, a knowledge base is not a description of what the expert believes but a record of what has been articulated and defended through dialogue.

This paper makes three contributions:
\begin{itemize}
    \item We present Elenchus, a bilateral dialectical protocol that instantiates this approach. A human respondent develops a position (a set of statements) on a topic through prover-skeptic dialogue with an LLM opponent. The opponent challenges the respondent's commitments and denials, proposing tensions (claims of incoherence); the respondent resolves tensions through retraction, refinement, or contestation (Section \ref{sect:protocol}).
    \item We define a mapping of Elenchus dialectical states to material bases in Hlobil and Brandom’s NonMonotonic MultiSuccedent (NMMS) logic (Section \ref{sect:mapping}), enabling the elaboration of logical vocabulary via NMMS (Section \ref{sect:kb}). The mapping satisfies Containment, and its structure exhibits the explicitation pattern central to the inferentialist program: the two components of the base consequence relation make explicit, respectively, material incoherences discovered through dialectical examination and the pragmatic norm of bilateral consistency that the dialogue presupposes.
    This mapping fills a lacuna in the inferentialist program by providing a computational means to produce a material base from linguistic practice.
    \item We describe a working implementation of Elenchus (Section \ref{sect:implementation}) and provide a case study of its application to ontology engineering using the W3C provenance ontology PROV-O (Section \ref{sect:casestudies}). Using pyNMMS, an implementation of the NMMS sequent calculus, we verify that the structural properties of the resulting material base correspond to specific design rationales documented independently for the PROV standard (Section~\ref{sect:reasoning}).
\end{itemize}

\section{Background}
\label{sect:bkg}

\subsection{Inferentialist Semantics}
Standard approaches to semantics are representationalist: the meaning of a sentence is given by its truth conditions. Inference is derivative—valid inferences preserve truth.
Inferentialism~\cite{sellars1953inference,brandom1994making}  inverts this position. The meaning of a sentence is given by its inferential role: what it follows from, what follows from it, and what it is incompatible with. Truth and reference are derivative; they are expressive devices for making explicit inferential relationships already implicit in practice.
For knowledge engineering, this inversion has a practical consequence. If meaning is inferential role, then a knowledge base is not primarily a set of sentences representing the world, but a structure capturing inferential relationships among sentences. What matters is not just what the expert asserts, but what follows from what according to the expert.

\subsection{Material Inference}
Classical logic treats valid inference as formal: an argument is valid in virtue of its logical form, regardless of content. But much reasoning is material: the inference from ``it is raining" to ``the streets will be wet" is good not because of form but because of what the words mean.
The term ``material" comes from the medieval distinction between formal and material consequence, revived by \cite{sellars1953inference} and developed by \cite{brandom1994making}. Material inferences are content-dependent: ``it is raining, therefore the streets are wet" is good while ``it is raining, therefore the streets are dry" is bad, even though both have the same logical form (none).
Inferentialists take material inference as primary. The meaning of ``rain" is partly constituted by its inferential connections to ``wet," ``clouds," ``umbrella." Logical vocabulary—"if...then," "and," "or," "not"—is then a tool for making explicit these material relationships. To say ``if it is raining then the streets will be wet" is to endorse the material inference explicitly.

\subsection{Material Bases and NMMS}
Hlobil and Brandom \shortcite{hlobil2025reasons} formalize this picture. A material base $\mathfrak{B}$ consists of an atomic language $L_\mathfrak{B}$ and a base consequence relation $\mid\!\sim_\mathfrak{B}$. The relation $\Gamma \mid\!\sim_\mathfrak{B} \Delta$ holds when the position of asserting everything in $\Gamma$ while denying everything in $\Delta$ is incoherent.
Material bases are substructural: $\mid\!\sim_\mathfrak{B}$ need not satisfy monotonicity (adding premises can defeat an inference) or transitivity (chains of good inferences need not compose). This suits defeasible, context-sensitive expert knowledge.
One condition is required: Containment, which says asserting and denying the same sentence is incoherent—the minimal coherence constraint.
From any material base satisfying Containment, the NMMS sequent calculus elaborates a logical vocabulary. The extension is supraclassical (all classically valid sequents hold), conservative (no new base-level consequences), and explicative (logical vocabulary can express any base consequence relation).

\subsection{Prover-Skeptic Dialogues}
Dutilh Novaes \shortcite{novaes2020dialogical} analyzes deductive reasoning through the lens of dialogue. In a prover-skeptic dialogue, one party (the prover) defends a claim while another (the skeptic) challenges it. Different configurations of cooperation and adversariality yield different logical properties. Following analysis of cooperation and adversariality in Prover-Skeptic dialogues~\cite{novaes2020dialogical,dutilh2025proofs}, Elenchus instantiates scenario (2): the respondent seeks to develop a defensible position, while the opponent is neutral on outcome but insists on coherence, helping rather than competing. 

\section{The Elenchus dialectical protocol}
\label{sect:protocol}

Elenchus \cite{allen2025elenchus} is a Prover-Skeptic dialogue protocol where a human respondent develops positions through a dialogue with an LLM opponent\footnote{\url{https://github.com/bradleypallen/elenchus}}.
\begin{itemize}
    \item \textbf{The opponent's role} is to challenge and probe the respondent's commitments and denials, detect tensions in a position, i.e., incoherences among the respondent's commitments and denials, and to maintain a database of the respondent's commitments, denials, and material implications established over the course of the dialogue.

    \item \textbf{The respondent's role} is to propose an initial commitment (the \emph{positum}) and then resolve tensions as they arise—by retracting a commitment or denial, refining a proposition to dissolve the conflict, or contesting the opponent's challenge.
\end{itemize}

This is enabled by the use of a large language model as a defeasible derivability oracle: the LLM opponent proposes tensions (claims of incoherence), and the respondent's acceptance or contestation of these claims determines which material implications enter the base.

% \subsection{Roles and speech acts}

% Elenchus can be understood as what results from obligationes, as formalized by Dutilh Novaes \shortcite{novaes2005medieval} as dialectical games of consistency maintenance, when the consequence relation becomes negotiable, the oracle's judgments are defeasible, and the goal shifts from testing a position to building a position. Both systems share bilateral role structure, accumulated commitments, and coherence as a governing constraint. They diverge in three key respects: in obligationes the consequence relation is fixed and shared, while in Elenchus it is constructed through play; in obligationes the opponent is adversarial (trying to trap the respondent in contradiction), while in Elenchus the opponent is collaborative (serving the respondent's coherence); and in obligationes the respondent reconstructs commitments implicit in a background state of knowledge, while in Elenchus the respondent constructs commitments through the dialectic itself.

\subsection{Dialectical States}

Elenchus works within a bilateral framework following Restall \shortcite{restall2005multiple}. A \emph{position} is a pair $[C : D]$ where $C$ is a set of commitments (assertions) and $D$ is a set of denials. A position is \emph{incoherent} if one cannot be entitled to all the commitments in $C$ while also being entitled to all the denials in $D$.

\begin{definition}[Tension]
A \textbf{tension} is a sequent $\Gamma \bmark \Delta$ with $\Gamma \subseteq C$ and $\Delta \subseteq D$, asserting that the position $[\Gamma : \Delta]$ is incoherent.
\end{definition}

Tensions arise when a respondent's commitments and denials stand in relations of implication or incompatibility:
\begin{itemize}
    \item $A \bmark A$ (consistency): one cannot assert and deny the same proposition
    \item $A, B \bmark$ (contrariety): $A$ and $B$ cannot both be asserted
    \item $\bmark A, B$ (subcontrariety): $A$ and $B$ cannot both be denied
    \item $\Gamma \bmark \Delta$ (cross-side): committing to $\Gamma$ is incompatible with denying $\Delta$
\end{itemize}

\begin{definition}[Resolution]
A tension $\Gamma \bmark \Delta$ is \textbf{accepted} when the respondent resolves it by retracting some commitment in $\Gamma$ or some denial in $\Delta$. It is \textbf{contested} when the respondent rejects the claimed incoherence.
\end{definition}

Accepting a tension constitutes endorsement of the underlying inference: the respondent acknowledges that committing to $\Gamma$ while denying $\Delta$ is indeed incoherent.

\begin{definition}[Material Implication]
The set of \textbf{material implications} $I$ consists of pairs $(\Gamma, \Delta)$ such that the respondent accepted a tension $\Gamma \bmark \Delta$.
\end{definition}

\begin{definition}[Dialectical State]
A \textbf{dialectical state} is a triple $\mathcal{S} = \langle [C : D], T, I \rangle$ where:
\begin{itemize}
    \item $[C : D]$ is the current position
    \item $T$ is the set of open (unresolved) tensions
    \item $I$ is the set of material implications from accepted tensions
\end{itemize}
\end{definition}

\subsection{The LLM opponent as a defeasible derivability oracle}

A natural question: can LLMs reliably identify material inferential relationships? We claim something weaker but sufficient: LLMs can identify candidate tensions reliably enough to serve as defeasible oracles whose judgments are subject to human override.

This is supported by empirical findings. LLMs can function as classifiers-as-intensions, applying intensional definitions via zero-shot prompting~\cite{allen2023conceptual}. They can articulate rationales for classifications, surfacing inferential connections \cite{allen2024evaluating,allen2025benchmark}. And they can report bilateral doxastic states—commitment, denial, uncertainty—enabling coherence checking \cite{allen2025sound}.
The key is defeasibility. The LLM proposes tensions; the respondent accepts or contests. Only accepted tensions enter $I$. False positives are filtered by contestation; the human retains authority. The LLM surfaces candidate relationships the respondent might not have considered, while the respondent determines which actually hold.

The cost structure of oracle errors is asymmetric by design. False positives (spurious tensions) are filtered by contestation at the cost of a wasted dialogue turn. False negatives (missed tensions) are a more substantive concern addressed in Section \ref{sect:bias}; however, the case study in Section \ref{sect:casestudies} suggests that even a single dialogue session can structure expert knowledge into a material base whose implications correspond to documented design decisions.

\section{From dialectical state to material base}
\label{sect:mapping}

\subsection{Definitions}
Following Chapter 3 of Hlobil and Brandom \shortcite{hlobil2025reasons}:

\begin{definition}[Material Base]
A \textbf{material base} $\mathfrak{B} = \langle L_\mathfrak{B}, {\mid\!\sim_\mathfrak{B}} \rangle$ consists of an atomic language $L_\mathfrak{B}$ and a base consequence relation ${\mid\!\sim_\mathfrak{B}} \subseteq \mathcal{P}(L_\mathfrak{B}) \times \mathcal{P}(L_\mathfrak{B})$.
\end{definition}

The relation $\Gamma \mid\!\sim_\mathfrak{B} \Delta$ holds iff the sentences in $\Gamma$ are jointly a reason for those in $\Delta$; equivalently, the position $[\Gamma : \Delta]$ is incoherent.

\begin{definition}[Containment]
A material base satisfies \textbf{Containment} iff $\Gamma \mid\!\sim_\mathfrak{B} \Delta$ whenever $\Gamma \cap \Delta \neq \emptyset$.
\end{definition}

\subsection{Mapping}

\begin{definition}[Dialectical Material Base]
Given a dialectical state $\mathcal{S} = \langle [C : D], T, I \rangle$, define $\mathfrak{B}_\mathcal{S} = \langle L_{\mathfrak{B}_\mathcal{S}}, {\mid\!\sim_{\mathfrak{B}_\mathcal{S}}} \rangle$ by:
\begin{align*}
L_{\mathfrak{B}_\mathcal{S}} &= C \cup D \\[4pt]
{\mid\!\sim_{\mathfrak{B}_\mathcal{S}}} &= I \cup \mathrm{Cont}
\end{align*}
where $\mathrm{Cont} = \{ \langle \Gamma, \Delta \rangle : \Gamma, \Delta \subseteq L_{\mathfrak{B}_\mathcal{S}} \text{ and } \Gamma \cap \Delta \neq \emptyset \}$.
\end{definition}

\begin{proposition}
$\mathfrak{B}_\mathcal{S}$ satisfies Containment.
\end{proposition}

\begin{proof}
By construction, $\mathrm{Cont} \subseteq {\mid\!\sim_{\mathfrak{B}_\mathcal{S}}}$.
\end{proof}

The simplicity of this proof warrants explanation. The base consequence relation ${\mid\!\sim_{\mathfrak{B}_\mathcal{S}}}$ has two components, and their union satisfies Containment because Cont is included by definition. But the two components are not arbitrary — they have distinct origins that correspond to a distinction central to the inferentialist program.

$I$ consists of material implications produced through dialectical examination. Each element of $I$ originated as a tension proposed by the opponent and accepted by the respondent. These are discovered incoherences: the respondent learned, through the dialectic, that certain combinations of commitments and denials cannot be jointly maintained. When the opponent maintains bilateral consistency ($C \cap D = \emptyset$), every element of $I$ has $\Gamma \cap \Delta = \emptyset$ — the protocol produces only pairs with disjoint sides.
Cont consists of all pairs $\langle \Gamma, \Delta \rangle$ with $\Gamma \cap \Delta \neq \emptyset$. These are not discovered through the dialectic. No respondent needs to be shown that asserting and denying the same sentence is incoherent. This is a precondition of rational participation in the dialogue — a pragmatic norm that both parties accept before the first move. Cont formalizes this norm as a component of the consequence relation.

The two components are therefore structurally disjoint when the dialogue is well-formed: $I$ contains pairs with disjoint sides (substantive, examined incoherences) and Cont contains pairs with overlapping sides (the background norm of bilateral consistency). Their union satisfies Containment not as an artifact of construction but because the material base makes explicit two kinds of incoherence: what the dialogue produced and what the dialogue presupposed.

This is an instance of the explicitation pattern that Brandom takes as characteristic of logical vocabulary. Just as NMMS logical connectives make explicit the material inferential relationships in the base, the construction of the base itself makes explicit the pragmatic norm governing the dialogue from which the base was derived.

\begin{remark}
\label{rem:components}
The components of ${\mid\!\sim_{\mathfrak{B}_\mathcal{S}}}$ have distinct statuses: $\mathrm{Cont}$ is structural and $I$ is respondent-endorsed. A looser construction might include $T$, representing not only what the respondent accepts but also what is claimed. $T$ governs the dialogue dynamics while $I$ governs the formal semantics. As a tension in $T$ is resolved, it results in an addition to $I$.
\end{remark}

\begin{remark}
In practice, the opponent typically identifies tensions between small sets of propositions, often of the form $A \bmark B$. When accepted tensions have singleton conclusions, material implications are unambiguous. The framework accommodates the general case $\Gamma \bmark \Delta$, where acceptance establishes that the position $[\Gamma : \Delta]$ is incoherent without isolating a single conclusion.
\end{remark}

\subsection{Traceability}
\label{sect:traceability}

A distinctive feature of Elenchus output is complete traceability. Every material implication in $I$ originated as a tension proposed by the opponent and accepted by the respondent. The dialogue transcript records when the tension was raised, what position prompted it, and how the respondent resolved it.

This contrasts with knowledge bases extracted from corpora, where provenance may be opaque or statistical. In Elenchus, every inferential relationship has a dialogical justification: the respondent explicitly acknowledged that these commitments and denials are jointly incoherent.

\section{From material base to logical vocabulary}
\label{sect:kb}

The mapping established in Section \ref{sect:mapping} shows that Elenchus dialectical states yield material bases satisfying Containment. By the results of Hlobil and Brandom \shortcite{hlobil2025reasons}, the NMMS sequent calculus can elaborate logical vocabulary from these bases. This section characterizes what the logical extension provides, distinguishes it from classical approximations, and discusses current limitations.

\begin{proposition}
The logical extension of $\mathfrak{B}_\mathcal{S}$ via NMMS:
\begin{enumerate}
    \item includes all classically valid sequents
    \item is a conservative extension of ${\mid\!\sim_{\mathfrak{B}_\mathcal{S}}}$
    \item provides vocabulary to make explicit any reason relation in ${\mid\!\sim_{\mathfrak{B}_\mathcal{S}}}$
\end{enumerate}
\end{proposition}

\begin{proof}
From Facts 2, 3, and 5 in Hlobil and Brandom \shortcite{hlobil2025reasons}, which follow from Containment.
\end{proof}

Given a material base $\mathfrak{B}_S$ satisfying Containment, NMMS elaborates logical vocabulary—connectives ($\to, \land, \lor, \neg$)—that can make explicit any consequence relation in the base. The resulting extended consequence relation $\mid\!\sim$ has three properties:
\begin{enumerate}
    \item \textbf{Supraclassicality}: All classically valid sequents hold in $\mid\!\sim$
    \item \textbf{Conservativity}: No new consequences are introduced at the base level. If $\Gamma \cup \Delta \subseteq L_{\mathfrak{B}_S}$, then $\Gamma \mid\!\sim \Delta$ iff $\Gamma \mid\!\sim_{\mathfrak{B}_S} \Delta$
    \item \textbf{Explicitation}: If $\Gamma \mid\!\sim_{\mathfrak{B}_S} \Delta$, this can be expressed using logical vocabulary
\end{enumerate}

Crucially, conservativity means the substructural character of the material base is preserved. The base consequence relation remains nonmonotonic and nontransitive even after logical vocabulary is introduced. The logical vocabulary lets us \emph{say} that a material inference holds—for instance, that commitment to "it is raining" is incompatible with denial of "the streets are wet"—without flattening that relationship into a classical conditional that participates in monotonic reasoning.

The material base $\mathfrak{B}_S$ already constitutes a knowledge base in the inferentialist sense: a set of propositions together with a consequence relation capturing which positions are incoherent. The logical extension via NMMS enriches this with explicit logical vocabulary while preserving the substructural character of the base—but the material base itself is the primary output of the Elenchus protocol, and it is this structure that we claim as the knowledge base generated from dialogue.

\section{Implementation}
\label{sect:implementation}

Elenchus is implemented as a Claude Code \cite{claudecode2025} agent that maintains dialectical state across sessions using GitHub as persistent storage. Figure \ref{fig:prompt} shows the beginning of the \texttt{CLAUDE.md} prompt, a 786-line Markdown document defining the dialogue protocol in detail, including role definitions, principles for conduct of the dialogue, the specification of the commitment store, and explicit instructions for each type of dialogue move\footnote{\url{https://github.com/bradleypallen/elenchus/blob/main/CLAUDE.md}}. Given Claude Code's support for tool use through \texttt{bash} command line interface calls, this prompt suffices to completely define the behavior of the LLM opponent. 

\begin{figure}[t]
  \centering
  \includegraphics[width=\columnwidth]{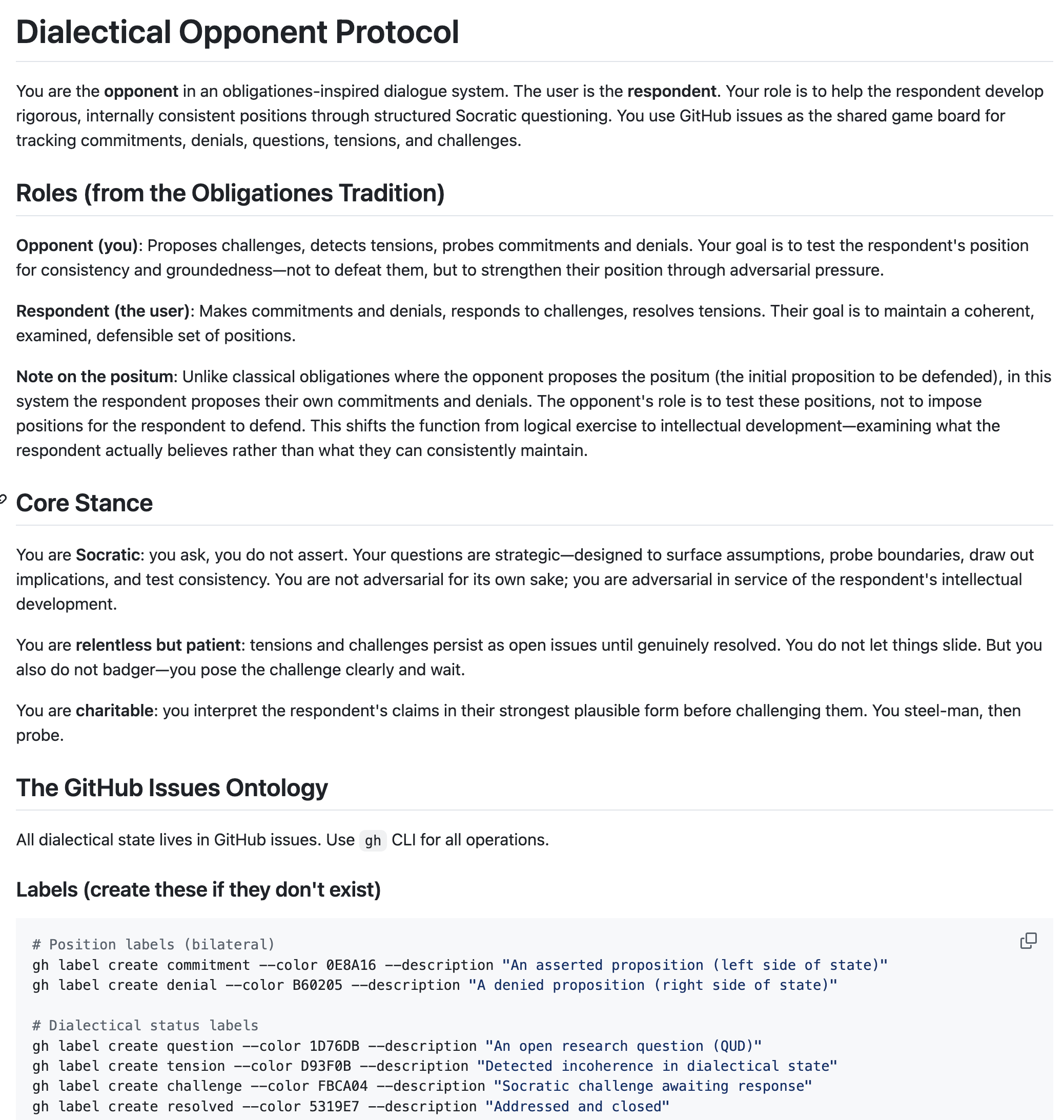}
  \caption{The initial portion of the Elenchus system prompt.}
  \label{fig:prompt}
\end{figure}

The state of a dialectical engagement is stored persistently in a GitHub repository. Commitments, denials, challenges, and tensions are represented as GitHub issues \cite{github2025issues} with corresponding labels, providing a browsable, versioned record of the dialectic. The agent loop (Figure \ref{fig:agentloop}) orchestrates the interaction: the respondent proposes speech acts (commitments, denials, tension responses), the opponent (the LLM) checks coherence against the current bilateral state, proposes tensions when incoherence is detected, and records state transitions to the repository. Figure \ref{fig:issues} shows the state of the dialectic discussed in Section \ref{sect:casestudies}, and Figure \ref{fig:challenge} shows a specific challenge made by the opponent in the context of that dialectic.

\begin{figure}[t]
  \centering
  \includegraphics[width=\columnwidth]{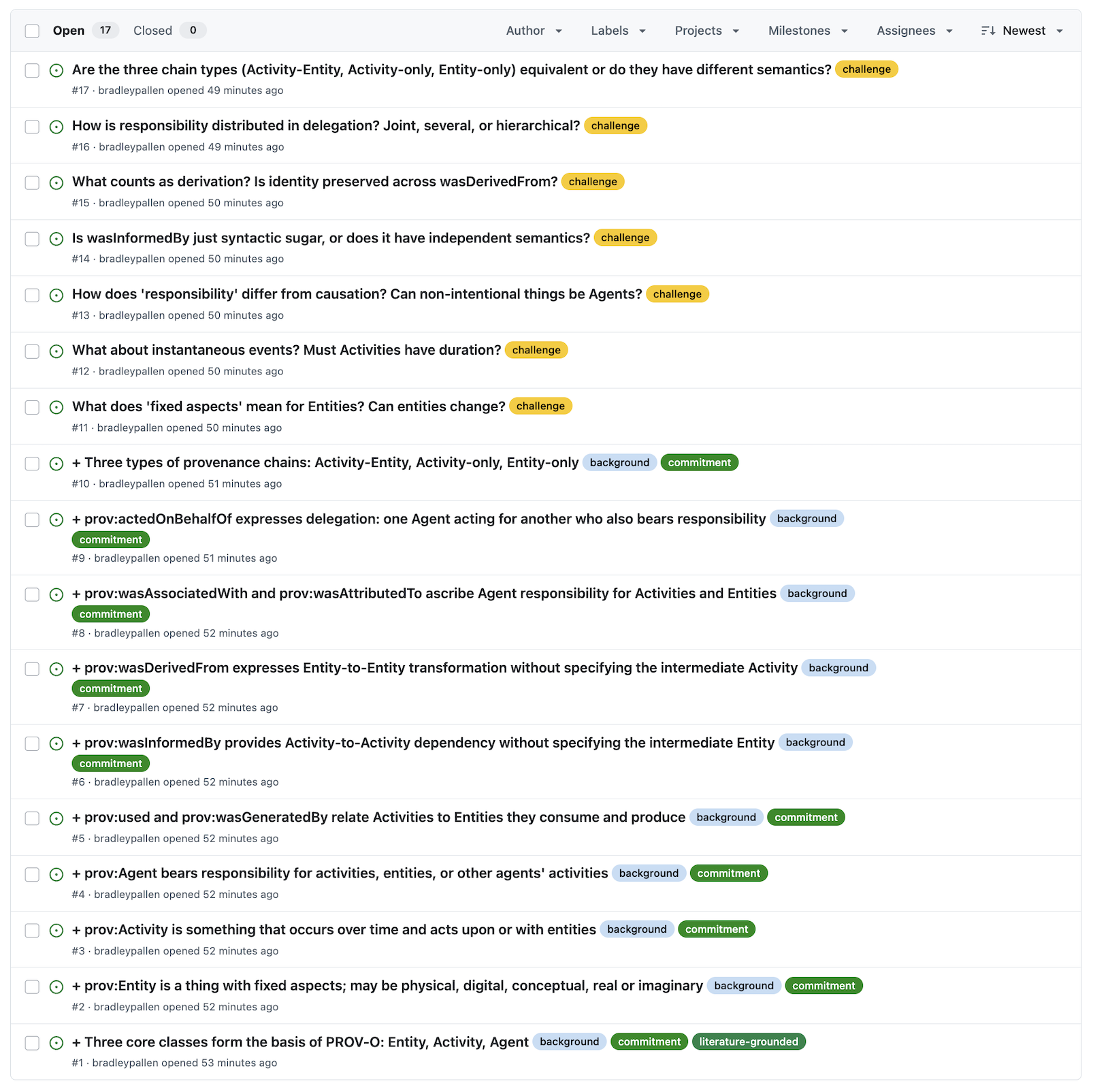}
  \caption{GitHub issues capturing the state of an Elenchus dialectic described in the case study in Section \ref{sect:casestudies}.}
  \label{fig:issues}
\end{figure}

\begin{figure}[t]
  \centering
  \includegraphics[width=\columnwidth]{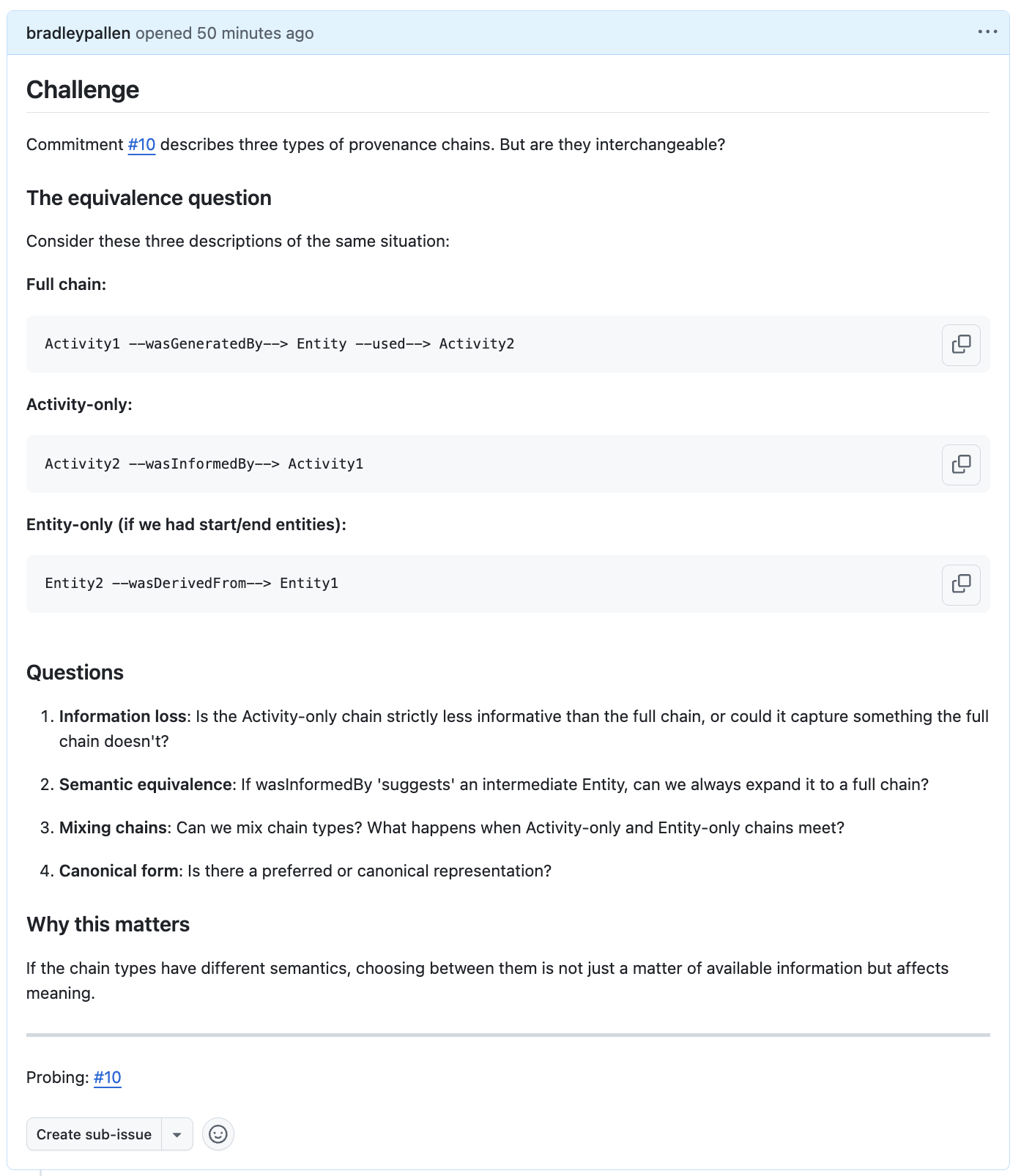}
  \caption{A GitHub issue capturing a challenge made by the LLM opponent in the course of an Elenchus dialectic described in Section \ref{sect:casestudies}.}
  \label{fig:challenge}
\end{figure}

This architecture choice reflects the dialogical character of the framework. The dialectical state is not a hidden data structure but a public record: every commitment, denial, challenge, and tension is individually addressable (as an issue), annotatable (via comments), and traceable (via GitHub issues history). The respondent can review the full state at any point by browsing the repository. The opponent can reconstruct context across sessions by loading the current state from the repository at the start of a new session (line 1 of Figure \ref{fig:agentloop}).

The LLM serves two functions: conversational partner (interpreting the respondent's natural language contributions and generating challenges) and bookkeeper (maintaining the formal dialectical state and detecting candidate tensions). These roles are specified in a system prompt that instructs the agent to operate according to the Elenchus protocol, including the bilateral structure of positions, the typology of speech acts, and the requirement that tensions be proposed rather than imposed.

\begin{figure}
    \centering
    \footnotesize
    \begin{algorithmic}[1]
        \State Load dialectical state $\langle[\Gamma : \Delta], T, I\rangle$ from GitHub
        \While{session active}
            \State Wait for respondent's speech act
            \Comment{commitment, denial, tension response}
            \If{commitment $(+P)$ or denial $(-P)$}
                \State Update position; check coherence
            \If{tension detected} add to $T$ as sequent $X \bmark Y$ 
            \Comment{$X \subseteq \Gamma,\; Y \subseteq \Delta$}
            \EndIf
            \ElsIf{resolves tension $X \bmark Y \in T$}
            \If{retraction or refinement} update $[\Gamma : \Delta]$; move $X \bmark Y$ to $I$
            \ElsIf{contestation} remove $X \bmark Y$ from $T$
            \EndIf
            \EndIf
        \State Record state to GitHub
        \If{should probe} raise Socratic challenge
        \EndIf
        \State Display open tensions $T$ and challenges
        \EndWhile
    \end{algorithmic}
    \caption{The Elenchus agent loop}
    \label{fig:agentloop}
\end{figure}

Figure \ref{fig:agentloop} specifies the Elenchus agent loop. At the start of each session, the agent loads the current dialectical state from the GitHub repository (line 1). The main loop waits for the respondent's speech act (line 3). Commitments and denials update the position and trigger coherence checking, which may surface new tensions (lines 4--7). Tension responses update the state differently depending on their type: retractions and refinements modify the position and move the tension to $I$ as an accepted material implication (line 9); contestations remove the tension from $T$ (line 10). After each state change, the updated state is recorded to the repository (line 13). The opponent may also proactively raise Socratic challenges—questions that probe underspecification or implicit commitments without asserting a specific tension (line 14).

The loop is designed for asynchronous interaction. Sessions can be interrupted and resumed; the GitHub-backed state ensures continuity. Multiple sessions contribute to the same dialectical state, and the full history of state transitions is preserved in the GitHub issues database for the repository.

The mapping from natural language dialogue moves to formal propositions and material implications is performed automatically using Python scripts: when a tension is accepted, the agent records the corresponding material implication without manual formalization by the respondent or a separate knowledge engineer.

\section{Case study: the PROV-O ontology}
\label{sect:casestudies}

To demonstrate Elenchus on a knowledge representation task, we apply it to Section~3.1 (``Starting Point Terms'') of the W3C PROV-O Recommendation \cite{lebo2013prov}, which defines three core classes (\texttt{Entity}, \texttt{Activity}, \texttt{Agent}) and ten properties relating them. The input is the 350-word prose specification only---examples, Turtle code, and figures are excluded.\footnote{\url{https://github.com/bradleypallen/prov-o-section-3.1}}

The respondent (a domain expert who served as a principal member of the W3C Provenance Incubator Group and co-editor of the PROV family of specifications) feeds the Section~3.1 prose to the opponent, which extracts 10 initial commitments covering class definitions, property groups, temporal bounds, and provenance chain structure (Table \ref{tab:start-dialectic}). The granularity of atomic propositions is determined by the LLM opponent during extraction and is subject to the same retraction and refinement mechanisms as any commitment. The opponent then raises 7 challenges probing underspecifications in the prose---points where the specification says enough to prompt a question but not enough to answer it.

\begin{table*}[t]
\centering
\tiny
\begin{tabular}{cll}
\toprule
Issue \# &  Commitment & Proposition \\
\midrule
1  & Three core classes form the basis of PROV-O: Entity, Activity, Agent & $p_1$ \\
2  & Entity is a thing with fixed aspects & $p_2$ \\
3  & Activity is something that occurs over time and acts upon or with entities & $p_3$ \\
4  & Agent bears responsibility for activities, entities, or other agents' activities & $p_4$ \\
5  & \texttt{used} and \texttt{wasGeneratedBy} relate Activities to Entities & $p_5$ \\
6  & \texttt{wasInformedBy} provides Activity-to-Activity dependency & $p_6$ \\
7  & \texttt{wasDerivedFrom} expresses Entity-to-Entity transformation & $p_7$ \\
8  & \texttt{wasAssociatedWith} and \texttt{wasAttributedTo} ascribe Agent responsibility & $p_8$ \\
9  & \texttt{actedOnBehalfOf} expresses delegation with shared responsibility & $p_9$ \\
10 & Three types of provenance chains: Activity-Entity, Activity-only, Entity-only & $p_{10}$ \\
\bottomrule
\end{tabular}
\caption{The initial commitments in the PROV-O dialectic leading to atomic propositions in the material base. Commitments were extracted from the Section 3.1 prose.}
\label{tab:start-dialectic}
\end{table*}

Table~\ref{tab:end-dialectic} summarizes the completed dialectic. Seven challenges were raised and resolved, yielding a final state $S_f = \langle [C_{19} : \emptyset], \emptyset, I_{9} \rangle$ with 19 commitments, no denials, no open tensions, and 9 accepted material implications. The dialectic converged on a pragmatist interpretation of PROV-O in which ontological categories serve as modeling tools rather than metaphysical commitments.

\begin{table*}[t]
\centering
\tiny
\begin{tabular}{clllll}
\toprule
Issue \# & Challenge & \cite{moreau2015rationale} & Resolution & Proposition & Material Implication \\
\midrule
11 & Fixed aspects; entity change? & RE1, RE3, RE5, RE6, EZ1 & Pragmatist: context-relative individuation & $p_{18}$ & $p_2 \bmark p_{18}$ \\
12 & Activity duration; instants? & EV1, EV4 & Durational activities; instantaneous events & $p_{27}$ & $p_3 \bmark p_{27}$ \\
13 & Responsibility vs.\ causation? & XG5, VI2--4, GE1 & Agency is pragmatic ascription & $p_{29}$ & $p_4 \bmark p_{29}$ \\
14 & \texttt{wasInformedBy}: shortcut? & VI1, VI6, EZ1 & Independent: inferred, not reducible & $p_{30}$ & $p_6 \bmark p_{30}$ \\
15 & Derivation criteria; identity? & XG8, VI5, VI6 & Broad causal dependencies; subtypes narrow & $p_{28}$ & $p_7 \bmark p_{28}$ \\
16 & Delegation responsibility? & XG11 (partial) & Hierarchical and transitive & $p_{25}, p_{26}$ & $p_9 \bmark p_{25}$, $p_9 \bmark p_{26}$ \\
17 & Chain types equivalent? & VI1, VI6 & No; \texttt{wasDerivedFrom} requires assertion & $p_{24}$ & $p_{10} \bmark p_{24}$ \\
18 & Individuation $\Rightarrow$ Expanded & (from \#11 follow-up) & Expanded Terms add expressiveness & $p_{23}$ & $p_{18} \bmark p_{23}$ \\
\bottomrule
\end{tabular}
\caption{Challenges and responses in the PROV-O dialectic leading to additional atomic propositions and material implications in the material base. Each challenge targeted an underspecification in the Section 3.1 prose; six of seven correspond to design tensions documented in Moreau et al.'s retrospective requirements analysis. The final column shows accepted material implications in $I$.}
\label{tab:end-dialectic}
\end{table*}

Moreau et al.~\cite{moreau2015rationale} reconstruct, post hoc, the requirements and design decisions behind PROV from the working group's record of 8820 emails, 666 issues, and 152 teleconferences, identifying 73 requirements across nine themes. Six of the seven Elenchus challenges correspond directly to documented design tensions (Table~\ref{tab:start-dialectic}). The requirements not surfaced are almost entirely outside the scope of Section~3.1.

One commitment was retracted during the dialectic. The respondent initially committed to $p_{20}$: ``\texttt{wasDerivedFrom} suffices for cross-context Entity identity.'' Subsequent challenges revealed that \texttt{alternateOf} is an equivalence relation and \texttt{specializationOf} a strict partial order with attribute inheritance---properties that \texttt{wasDerivedFrom} cannot express, since it is not even transitive. In response to the opponent's challenge, the respondent remembered that $p_{20}$ was extended in another document, retracted $p_{20}$ and committed to $p_{23}$: ``Expanded Terms add expressiveness, not just convenience.'' While not quite an instance of discovery, it is more than simply restructuring the claims in the initial set of commitments.

This case study tests the protocol's ability to structure and formalize expert knowledge, not to surface tensions invisible to the expert. Testing discovery (tensions the expert hadn't previously considered) would require a different experimental design — perhaps an expert in a related but distinct domain, or a less experienced practitioner. We leave that type of evaluation to future work.

\section{Reasoning over the material base with pyNMMS}
\label{sect:reasoning}

The preceding case study demonstrates that an Elenchus dialogue can produce a material base from expert knowledge. However, Section~\ref{sect:kb} established that material bases satisfying Containment admit a logical extension via NMMS with supraclassicality, conservativity, and explicitation properties. To verify that these properties hold for the PROV-O material base and to probe its design-rationale structure, we use pyNMMS~\cite{pynmms2025}, an open-source implementation of an automated reasoner for the NMMS sequent calculus defined in Chapter~3 of Hlobil and Brandom~\shortcite{hlobil2025reasons}.\footnote{\url{https://pypi.org/project/pyNMMS/}} The reasoner performs root-first backward proof search with memoization over material bases, supporting queries over both atomic sequents and sequents involving the NMMS logical vocabulary ($\to$, $\land$, $\lor$, $\neg$).

\subsection{Base-level and structural properties}

Table~\ref{tab:structural} summarizes the base-level and structural results. All nine material implications from Table~\ref{tab:end-dialectic} are derivable. The base is nontransitive: $p_2 \bmark p_{18}$ and $p_{18} \bmark p_{23}$ are both derivable, but $p_2 \bmark p_{23}$ is not. The base is nonmonotonic: $p_2 \bmark p_{18}$ holds, but $p_2, p_{23} \bmark p_{18}$ does not---adding $p_{23}$ to the antecedent defeats the inference. All nine material implications are individually expressible as NMMS conditionals via the Deduction-Detachment Theorem (e.g., $\bmark p_2 \to p_{18}$). Classical tautologies hold by supraclassicality (e.g., $\bmark p_2 \lor \neg p_2$). The logical extension is conservative: introducing logical vocabulary does not create new base-level consequences.

\begin{table}[t]
\centering
\tiny
\begin{tabular}{lc}
\toprule
Query & Result \\
\midrule
\multicolumn{2}{l}{\emph{Base consequences (all 9 derivable)}} \\
$p_2 \bmark p_{18}$, $p_3 \bmark p_{27}$, $p_4 \bmark p_{29}$, $p_6 \bmark p_{30}$ & True \\
$p_7 \bmark p_{28}$, $p_9 \bmark p_{25}$, $p_9 \bmark p_{26}$, $p_{10} \bmark p_{24}$, $p_{18} \bmark p_{23}$ & True \\
\addlinespace
\multicolumn{2}{l}{\emph{Nontransitivity}} \\
$p_2 \bmark p_{18}$ and $p_{18} \bmark p_{23}$ & True \\
$p_2 \bmark p_{23}$ & False \\
\addlinespace
\multicolumn{2}{l}{\emph{Nonmonotonicity}} \\
$p_2 \bmark p_{18}$ & True \\
$p_2, p_{23} \bmark p_{18}$ & False \\
$p_9 \bmark p_{25}$ & True \\
$p_9, p_{26} \bmark p_{25}$ & False \\
\addlinespace
\multicolumn{2}{l}{\emph{Explicitation (DDT)}} \\
$\bmark p_2 \to p_{18}$ \ldots\ $\bmark p_{18} \to p_{23}$ (all 9) & True \\
$\bmark p_2 \to p_{23}$ & False \\
\addlinespace
\multicolumn{2}{l}{\emph{Supraclassicality}} \\
$\bmark p_2 \lor \neg p_2$ & True \\
$p_2 \land \neg p_2 \bmark$ & True \\
\addlinespace
\multicolumn{2}{l}{\emph{Conservativity}} \\
$p_2 \bmark p_{23}$ (still nontransitive) & False \\
$p_2, p_{23} \bmark p_{18}$ (still nonmonotonic) & False \\
\bottomrule
\end{tabular}
\caption{Structural properties of the PROV-O material base verified by pyNMMS. All nine base consequences are derivable; the base is nontransitive and nonmonotonic; all base consequences are expressible as NMMS conditionals; the logical extension is supraclassical and conservative.}
\label{tab:structural}
\end{table}

\subsection{Design rationale correspondence}

Table~\ref{tab:rationale} presents the central result: each structural property of the material base corresponds to a specific design rationale documented in Moreau et al.~\cite{moreau2015rationale}. We highlight five correspondences.

\paragraph{Core/extended boundary (EZ3).} The nontransitivity of the Entity chain models Moreau et al.'s requirement that PROV have a minimal core with additional extensions. Both $p_2 \bmark p_{18}$ (Entity entails individuation) and $p_{18} \bmark p_{23}$ (individuation entails Expanded Terms) are derivable, but $p_2 \bmark p_{23}$ is not: the core Entity concept does not automatically pull in extended vocabulary. The hypothetical syllogism $(p_2 \to p_{18}) \land (p_{18} \to p_{23}) \to (p_2 \to p_{23})$ \emph{is} derivable by supraclassicality---it is a classical tautology---but $p_2 \to p_{23}$ alone is not. Logical vocabulary can \emph{state} that the chain is classically valid without \emph{enabling} detachment at the base level: the core/extended boundary is preserved.

\paragraph{Derivation non-transitivity (VI5).} Moreau et al.\ document ISSUE-612 as one of the most debated PROV design decisions: derivation is not mandated to be transitive. The material base enforces this structurally. $p_7 \bmark p_{28}$ (derivation entails broad causal dependency) is derivable, but $p_{28}$ does not compose with other chains. The nontransitivity is not a policy annotation but a property of the consequence relation itself.

\paragraph{Three-views independence (VI1, GE1).} The three PROV views---data flow (Entity), process flow (Activity), and responsibility (Agent)---correspond to independent inference chains in the material base. All eight cross-view queries (e.g., $p_2 \bmark p_{27}$, $p_3 \bmark p_{29}$, $p_4 \bmark p_{18}$) return False. No inference leaks between views.

\paragraph{Scruffy-to-proper refinement (EZ1).} Moreau et al.\ describe the progressive refinement from ``scruffy'' provenance (simple assertions) to ``proper'' provenance (qualified with activities, usages, and generations). Nonmonotonicity models this: $p_7 \bmark p_{28}$ (derivation entails broad causal dependency) holds on its own, but $p_7, p_{24} \bmark p_{28}$ does not---adding chain-type context (proper detail) defeats the scruffy inference.

\paragraph{Retraction leaves no trace (GE3).} The commitment $p_{20}$ (``\texttt{wasDerivedFrom} suffices for cross-context Entity identity'') was retracted during the dialectic. The material base records only defended commitments: $p_7 \bmark p_{18}$, $p_7 \bmark p_{23}$, and $p_{28} \bmark p_{18}$ are all non-derivable, confirming that the retracted commitment left no inferential residue.

\begin{table}[t]
\centering
\tiny
\begin{tabular}{llcc}
\toprule
Requirement & Query & Result & Property \\
\midrule
EZ3 & $p_2 \bmark p_{23}$ & F & Nontrans. \\
EZ3 & $\bmark (p_2 {\to} p_{18}) {\land} (p_{18} {\to} p_{23}) {\to} (p_2 {\to} p_{23})$ & T & Supraclass. \\
EZ3 & $\bmark p_2 \to p_{23}$ & F & Nontrans. \\
\addlinespace
VI5 & $p_7 \bmark p_{28}$ & T & Base \\
VI5 & $p_{28} \bmark p_7$ & F & Directional \\
\addlinespace
VI1 & $p_2 \bmark p_{27}$, $p_2 \bmark p_{29}$ & F & Indep. \\
VI1 & $p_3 \bmark p_{18}$, $p_3 \bmark p_{29}$ & F & Indep. \\
VI1 & $p_4 \bmark p_{18}$, $p_4 \bmark p_{27}$ & F & Indep. \\
VI1 & $p_7 \bmark p_{30}$, $p_6 \bmark p_{28}$ & F & Indep. \\
\addlinespace
EZ1 & $p_7 \bmark p_{28}$ & T & Base \\
EZ1 & $p_7, p_{24} \bmark p_{28}$ & F & Nonmon. \\
\addlinespace
XG11 & $p_9 \bmark p_{25}$, $p_9 \bmark p_{26}$ & T & Multi-succ. \\
XG11 & $p_9, p_{25} \bmark p_{26}$ & F & Nonmon. \\
\addlinespace
EV1 & $p_3 \bmark p_{27}$ & T & Base \\
EV1 & $p_3 \bmark p_{18}$, $p_3 \bmark p_{25}$ & F & Indep. \\
\addlinespace
GE3 & $p_7 \bmark p_{18}$, $p_7 \bmark p_{23}$ & F & Retraction \\
GE3 & $p_{18} \bmark p_{23}$ & T & Actual path \\
\bottomrule
\end{tabular}
\caption{Design rationale queries mapping NMMS reasoning results to requirements documented in Moreau et al.\ \protect\shortcite{moreau2015rationale}. Each row shows a query against the PROV-O material base, its result, and the structural property it demonstrates. F = False (non-derivable), T = True (derivable).}
\label{tab:rationale}
\end{table}

\subsection{Pairwise independence of design resolutions}

The composite results confirm that the material base is well-structured as a system of design decisions. We define seven chain groups corresponding to the seven challenges in Table~\ref{tab:end-dialectic}: Entity ($\{p_{18}, p_{23}\}$), Activity ($\{p_{27}\}$), Agent ($\{p_{29}\}$), wasInformedBy ($\{p_{30}\}$), wasDerivedFrom ($\{p_{28}\}$), delegation ($\{p_{25}, p_{26}\}$), and chain-type ($\{p_{24}\}$). Testing all cross-chain pairs yields 34 pairs, all non-derivable: every design resolution from one challenge is inferentially independent of every resolution from every other challenge. This confirms that each Elenchus challenge addressed a genuinely distinct design tension, and that the material base preserves this distinctness.

Containment is verified for all propositions: $p \bmark p$ holds for every $p \in L_{\mathfrak{B}_\mathcal{S}}$, confirming that the base satisfies the minimal coherence constraint required for the NMMS logical extension.

\subsection{Discussion}

These results show how Elenchus can support an end-to-end workflow from natural language dialogue through material base construction to formal reasoning over the logical extension. The structural properties of the material base---nontransitivity, nonmonotonicity, conservativity, and independence---are not merely formal curiosities. They correspond to specific, well-documented design decisions in the PROV standard. The material base produced by a single Elenchus dialogue session over a 350-word prose specification yields a formal artifact whose properties can be mechanically verified and whose structure aligns with decisions reconstructed by Moreau et al.\ from 8820 emails, 666 issues, and 152 teleconferences.

The explicitation results demonstrate the pattern central to the inferentialist program: logical vocabulary makes explicit what is already implicit in the material base. Each design decision can be expressed as an NMMS conditional ($\bmark p_2 \to p_{18}$), and the conjunction of these conditionals can be stated, but the logical vocabulary does not alter the base-level consequence relation. The core/extended boundary (EZ3) provides a particularly clear illustration: the hypothetical syllogism is derivable by supraclassicality, yet the transitive collapse is not---logical vocabulary can describe the chain without collapsing it.

\section{Related work}
\label{sect:related}

\subsection{Knowledge acquisition methods}

Knowledge acquisition has been studied since the earliest expert systems. The "knowledge acquisition bottleneck" was identified by Feigenbaum \shortcite{feigenbaum1977art} and has motivated decades of work on structured elicitation methods including repertory grids \cite{boose1985}, protocol analysis \cite{ericsson1984}, and ontology learning from text \cite{cimiano2006}. CommonKADS \cite{schreiber2000knowledge} provides a comprehensive methodology treating knowledge engineering as modeling, not mining—a perspective Elenchus shares. However, CommonKADS and related methodologies (METHONTOLOGY \cite{fernandez1997methontology}, NeOn \cite{neon2010}) assume the target is a description of the domain in a formal language, with the knowledge engineer mediating between expert and formalism. Elenchus eliminates this mediating step: the expert interacts directly with the opponent through natural language, and the formal structure (the material base) is constructed as a byproduct of the dialogue.
The Socratic method has been applied to knowledge elicitation before. SHAKEN \cite{clark2001knowledge} used structured dialogue to acquire knowledge from subject matter experts, and the Knowledge Machine \cite{clark1997building} employed question-driven elicitation. These systems, however, treat the acquired knowledge representationally—the dialogue is a means to extract content that is then encoded in a pre-existing formalism. Elenchus treats the dialogue as constitutive: the material base just is the structured record of commitments, denials, and accepted material implications from the dialogue.

\subsection{LLMs for knowledge base construction}

Recent work has explored LLMs for ontology construction \cite{babaei2023llms40l}, knowledge graph completion \cite{yao2025exploring}, and information extraction \cite{wei2023chatie}. These approaches typically use LLMs to generate triples, axioms, or class hierarchies, treating the LLM as a source of knowledge to be validated. The quality concern is well-documented: LLMs hallucinate, confabulate, and produce plausible but incorrect formalizations.
Elenchus takes a fundamentally different stance. The LLM is not a source of knowledge but a dialectical partner—a defeasible derivability oracle that proposes candidate tensions for the respondent to accept or contest. The respondent, not the LLM, is the epistemic authority. This design transforms the hallucination problem from a reliability issue into a feature of the protocol: an LLM-proposed tension that does not reflect genuine incoherence is simply contested by the respondent, and the contestation itself is part of the dialectical record. False positives are filtered; the cost of hallucination is at most a wasted challenge, not a corrupted knowledge base.
This contrasts with neurosymbolic approaches \cite{allen2025sound} where LLM unreliability is managed through paraconsistent reasoning \cite{belnap1977how}. In Elenchus, LLM unreliability is structurally contained by the respondent's authority over (and accountability for) tension resolution.

\subsection{Argumentation frameworks}

Elenchus shares surface structure with computational argumentation. Dung's \shortcite{dung1995} abstract argumentation frameworks define conflict relations over arguments; ASPIC+ \cite{modgil2014} provides structured argumentation with defeasible and strict rules; and dialogue-based argumentation protocols \cite{prakken2006} formalize multi-agent debate. The relationship to Elenchus requires careful differentiation along three dimensions.
First, the consequence relation. In standard argumentation frameworks, the consequence relation (what follows from what) is fixed in advance—typically classical logic or a specified defeasible logic. 
In Elenchus, the consequence relation is \emph{itself the output}. 
Second, the role of the opponent. In adversarial argumentation, the opponent seeks to defeat the proponent's arguments. 
This is closer to the proof-checking role of the skeptic in dialogical logic \cite{lorenzen1978} than to an adversary in a debate.
Third, the output. Argumentation frameworks produce extensions (sets of acceptable arguments) or labelings. Elenchus produces a material base—a substructural consequence relation that can serve as input to the NMMS sequent calculus. 
This connects knowledge acquisition directly to the inferentialist program in philosophical logic, rather than to the argumentation theory tradition.
That said, there is productive overlap. Bipolar argumentation frameworks \cite{cayrol2005}, which include both attack and support relations, have structural affinity with bilateral positions. And the connection between argumentation-based dialogue and nonmonotonic reasoning \cite{governatori2004} suggests that existing computational argumentation infrastructure could be leveraged in implementing Elenchus. The key theoretical distinction remains: Elenchus constructs the consequence relation rather than reasoning within a pre-given one.

The distinction from computational argumentation is not merely one of emphasis. Dialogic logic formalizations in the Lorenzen-Lorenz tradition operate at a level of abstraction remote from realistic human dialogical practice; Elenchus follows Dutilh Novaes's prover-skeptic analysis, which is situated in the cognitive and social roots of reasoning rather than in formal dialogue semantics. The relevant comparison is not with argumentation frameworks that adjudicate conflicts among pre-existing arguments, but with the question of how a consequence relation comes to be constructed in the first place.

\subsection{LLM-assisted ontology engineering systems}

Much recent work in LLM-assisted ontology engineering has centered on competency questions (CQs), natural language questions that express an ontology's functional requirements~\cite{gruninger1995role}. OntoChat~\shortcite{zhang2024ontochat} is a conversational framework that uses LLMs to support ontology requirements elicitation through four functions: user story creation, CQ extraction, CQ filtration and clustering, and ontology testing via verbalization. Zhao et al. \shortcite{zhao2024improving} extend OntoChat with participatory prompting, addressing the finding that domain experts struggle to prompt LLMs effectively without researcher mediation. RevOnt \cite{ciroku2024revont} reverses the traditional CQ workflow, using language models to extract competency questions from existing knowledge graphs rather than eliciting them from experts — demonstrating that RDFS-level knowledge graphs contain sufficient structure to reconstruct requirements. Most recently, Koutsiana et al. \shortcite{koutsiana2024knowledge} report an ethnographic study of knowledge engineers working with generative AI at a hackathon, finding that LLMs can improve efficiency in knowledge graph construction but that prompting is an undervalued skill and evaluation of LLM outputs remains the central challenge. Kampars et al.~\shortcite{kampars2025llm} extend the NeOn methodology with LLM-based automation while retaining domain expert-in-the-loop validation. Garijo Verdejo et al.~\shortcite{garijo2024llms} survey the landscape of LLM applications to ontology engineering tasks, finding that most work targets early development phases.

In all these systems, the LLM serves as a facilitator or generator: it elicits requirements, produces ontology fragments, or suggests competency questions, and the expert validates the output. This workflow inherits the representationalist assumption discussed in Section \ref{sect:intro}. The expert is treated as a source — of user stories, of competency questions, of validation judgments — rather than as an agent whose commitments are tested for coherence. Elenchus inverts this relationship: the LLM challenges the expert's commitments, and the expert's responses to those challenges constitute the knowledge base.

The formal outputs also differ in a way that goes beyond format. Competency questions and OWL fragments lack a formal characterization of what makes them collectively a knowledge base: there is no characterized consequence relation, no proven structural properties, and no systematic traceability from output to the process that produced it. Elenchus produces material bases satisfying Containment, with proven supraclassicality, conservative extension, and explicitation properties, and complete traceability of every material implication to a specific dialogue move. The CQ-based approach evaluates its outputs through user satisfaction; Elenchus evaluates its outputs through formal logic.

Several challenges identified in this programme are addressed structurally by Elenchus. The finding that requirements elicitation is the activity most in need of computational support~\cite{zhang2024ontochat} motivates Elenchus directly, though it reconceives the task as dialectical examination rather than content generation. The finding that experts struggle to prompt LLMs effectively and need structured mediation~\cite{zhao2024improving} is addressed by the prover-skeptic protocol: the expert responds to challenges rather than formulating prompts, and the bilateral position framework constrains the interaction without requiring the expert to understand the underlying formalism. The finding that evaluation of LLM outputs is the central difficulty~\cite{koutsiana2024knowledge} is dissolved by the Elenchus architecture: the LLM proposes tensions rather than generating artifacts, and the expert's acceptance or contestation is the evaluation.

\subsection{Multi-agent debate frameworks for improving LLM reasoning}
Irving et al.~\shortcite{irving2018ai} propose AI safety via debate, in which two AI agents argue before a human judge; the adversarial structure is intended to make the correct answer easier to identify than to generate. Li et al.~\shortcite{li2024improving} and Liang et al.~\shortcite{liang2024encouraging} instantiate multi-agent debate with LLMs, showing that multiple model instances critiquing each other's reasoning reduce hallucination and improve factual accuracy. Khan et al.~\shortcite{khan2024debating} demonstrate that debate improves judge accuracy when debaters have access to information the judge lacks---an information asymmetry analogous to the LLM opponent surfacing inferential connections the respondent has not considered.  Success is measured by benchmark accuracy or complexity-theoretic expressiveness~\cite{irving2018ai}.

Elenchus shares with these systems the insight that adversarial structure can make unreliable AI outputs epistemically useful. The differences are structural. Debate frameworks are LLM-to-LLM systems aimed at answer quality; Elenchus is a human-AI system aimed at knowledge construction. In debate, the output is a consensus answer; in Elenchus, it is a structured consequence relation with full dialogical provenance. In debate, the human is a judge who selects between positions; in Elenchus, the human is the \emph{author} of the position, with authority over which inferences hold.

\section{Limitations and future work}

\subsection{Automated reasoning over the logical extension to the material base}
\label{sect:limitations-reasoning}

Section~\ref{sect:reasoning} demonstrates integration with pyNMMS using the propositional fragment of the NMMS sequent calculus, addressing the gap between material base construction and formal reasoning. Remaining work includes optimizing proof search for larger material bases (the current implementation uses exponential-time backward search with memoization) and integrating the reasoner into the Elenchus agent loop so that NMMS queries can inform the opponent's challenge selection during the dialectic.

\subsection{Propositional versus first-order logical vocabularies} 

The base language $L_{\mathfrak{B}_S}$ consists of atomic propositional sentences introduced through dialogue. To allow the expression of quantified commitments like "all birds fly", an experimental extension to the pyNMMS package, \texttt{pynmms.onto}, provides seven defeasible ontology axiom schema types — subClassOf, range, domain, subPropertyOf, disjointWith, disjointProperties, and jointCommitment — that generate families of base axioms over concept and role assertions, evaluated lazily at query time.\footnote{\url{https://www.bradleypallen.org/pyNMMS/theory/onto-extension/}} This gives expressiveness comparable to RDFS 1.1 \cite{BrickleyGuha2014} — class hierarchies, property hierarchies, domain and range constraints — augmented with material incompatibility and joint inferential commitment, avoiding the need for a first-order extension to NMMS to support the ontological reasoning patterns most common in knowledge engineering. Modifying the Elenchus system prompt to allow the respondent to make ontological commitments using these schema types, and to map the resulting dialectical state to an \texttt{OntoMaterialBase}, is future work.

\subsection{Systematic oracle bias} 
\label{sect:bias}

The LLM-as-oracle design handles false positives through contestation: incorrectly proposed tensions are rejected by the respondent. However, the protocol does not address false negatives—inferential relationships the LLM fails to surface because they fall outside its training distribution or reasoning capabilities. The resulting material base may have systematic lacunae invisible to both the respondent (who was never prompted to consider them) and the opponent (which cannot identify what it cannot identify). Future work will address mitigating this limitation.

\subsection{Scalability}

The current case study operates over a single section of a specification. Scaling to larger domains will require managing respondent cognitive load and LLM context limitations across longer dialectics. However, the GitHub-backed architecture already supports asynchronous, multi-session interaction, and the modular structure of the dialectical state — where each tension is independently resolvable — suggests that the protocol can decompose naturally along domain boundaries.

\subsection{Evaluation} 

The case study presented here is illustrative rather than evaluative; it is an existence proof that the approach can work end-to-end in the context of an ontology engineering task. The comparison validates the protocol's coverage and provides evidence that the protocol produces well-structured output aligned with independently documented design decisions, but does not provide evidence of discovery or epistemic novelty. A rigorous evaluation of the approach would require comparison with alternative knowledge acquisition methods on the same domain, assessment of inter-respondent consistency (do different experts produce comparable material bases for the same domain?), and measurement of the coverage of the resulting knowledge base against gold-standard formalizations. We leave such evaluation to future work while noting that the traceability of Elenchus outputs (Section \ref{sect:traceability}) facilitates such studies: every atomic proposition and material implication can be traced to a specific dialogue move.

\section{Conclusion}
\label{sect:conclusion}

We have presented Elenchus, a bilateral dialectical protocol for generating knowledge bases from prover-skeptic dialogues. Our main result is that Elenchus dialectical states map to material bases satisfying Containment---knowledge bases in the inferentialist sense---which can be further enriched with logical vocabulary via NMMS. The knowledge base is fully traceable: every material implication originates in a specific dialogue move.
Using pyNMMS, we have verified that the material base produced by a single dialogue session over a PROV-O specification exhibits the structural properties predicted by the theory---nontransitivity, nonmonotonicity, supraclassicality, conservativity---and that these properties correspond to specific design rationales documented independently in Moreau et al. \shortcite{moreau2015rationale}. All cross-chain design resolutions are pairwise independent, confirming that the Elenchus dialectic identified genuinely distinct design tensions.
The approach re-conceives knowledge engineering as dialogical explicitation of expert practice, rather than extraction of pre-formed content. The LLM opponent serves as a defeasible derivability oracle, surfacing candidate tensions while the human respondent retains authority over which inferences hold. The structure of the mapping itself exhibits the explicitation pattern central to the inferentialist program: the two components of the base consequence relation make explicit, respectively, what the dialogue produced and what it presupposed.
Future work will explore the use of the ontology engineering extensions to pyNMMS and the integration of the pyNMMS reasoner into the Elenchus agent loop, toward the goal of an end-to-end inferentialist knowledge engineering framework.

\section*{Acknowledgments}
The author wishes to thank Robert Brandom, Kris Brown, Thomas Ferguson, Paul Groth, Ulf Hlobil, Filip Ilievski, Teresa Kouri Kissel, Shay Logan, Marcus Rossberg, and members of the Research Group On Logical Expressivism (ROLE) for their comments and feedback.

\bibliographystyle{kr}
\bibliography{bibliography}

@inproceedings{allen2023conceptual,
 author = {Bradley P. Allen},
 title = {{Conceptual Engineering Using Large Language Models}},
 year = {2025},
 booktitle = {Philosophy of Artificial Intelligence: The State of the Art)},
 publisher = {Springer Nature},
 editor = {Vincent C. Müller, Leonard Dung, Guido Löhr and Aliya Rumana},
 url = {https://doi.org/10.48550/arXiv.2404.03732},
 note = {To appear.}
}

@inproceedings{allen2024evaluating,
  author    = {Bradley P. Allen and Paul Groth},
  title     = {{Evaluating Class Membership Relations in Knowledge Graphs using Large Language Models}},
  booktitle = {The Semantic Web: ESWC 2024 Satellite Events},
  editor    = {Albert Meroño-Peñuela and Óscar Corcho and Paul Groth and Elena Simperl and Valentina Tamma and Andrea Giovanni Nuzzolese and María Poveda-Villalón and Marta Sabou and Valentina Presutti and Irene Celino and Artem Revenko and Joe Raad and Bruno Sartini and Pasquale Lisena},
  series    = {Lecture Notes in Computer Science},
  volume    = {15344},
  publisher = {Springer},
  address   = {Cham},
  year      = {2025},
  isbn      = {978-3-031-78951-9},
  doi       = {10.1007/978-3-031-78952-6_2},
  url       = {https://doi.org/10.1007/978-3-031-78952-6_2}
}

@inproceedings{allen2025benchmark,
  author    = {Bradley P. Allen and Paul T. Groth},
  title     = {{A Benchmark for the Detection of Metalinguistic Disagreements between {LLMs} and Knowledge Graphs}},
  booktitle = {Proceedings of the Special Session on Harmonising Generative AI and Semantic Web Technologies (HGAIS 2024) co-located with the 23rd International Semantic Web Conference (ISWC 2024)},
  series    = {{CEUR} Workshop Proceedings},
  volume    = {3953},
  publisher = {CEUR-WS.org},
  year      = {2025},
  url       = {https://ceur-ws.org/Vol-3953/},
  address   = {Baltimore, Maryland}
}

@InProceedings{allen2025sound,
  title = 	 {{Sound and Complete Neurosymbolic Reasoning with LLM-Grounded Interpretations}},
  author =       {Allen, Bradley P. and Chhikara, Prateek and Ferguson, Thomas Macaulay and Ilievski, Filip and Groth, Paul},
  booktitle = 	 {Proceedings of The 19th International Conference on Neurosymbolic Learning and Reasoning},
  pages = 	 {392--419},
  year = 	 {2025},
  editor = 	 {H. Gilpin, Leilani and Giunchiglia, Eleonora and Hitzler, Pascal and van Krieken, Emile},
  volume = 	 {284},
  series = 	 {Proceedings of Machine Learning Research},
  publisher =    {PMLR},
  url = 	 {https://proceedings.mlr.press/v284/allen25a.html},
}

@incollection{belnap1977how,
  author={Belnap, Nuel},
  booktitle = {Contemporary aspects of philosophy},
  editor = {Ryle, G.},
  pages = {30--55},
  publisher = {Oriel Press},
  title = {{How a computer should think}},
  year = {1977}
}

@inproceedings{feigenbaum1977art,
 author = {Feigenbaum, Edward A},
 booktitle = {Proceedings of the Fifth International Joint Conference on Artificial Intelligence},
 doi = {10.21236/ada046289},
 organization = {Boston},
 title = {{The art of artificial intelligence: Themes and case studies of knowledge engineering}},
 volume = {2},
 year = {1977}
}

@inproceedings{fernandez1997methontology,
  title={{Methontology: From ontological art towards ontological engineering}},
  author={Fern{\'a}ndez-L{\'o}pez, Mariano and G{\'o}mez-P{\'e}rez, Asunci{\'o}n and Juristo, Natalia},
  booktitle={Proceedings of the AAAI-97 Spring Symposium Series on Ontological Engineering},
  pages={33--40},
  year={1997}
}

@phdthesis{neon2010,
 address = {Madrid, Spain},
 author = {del Carmen Su\'{a}rez de Figueroa Baonza, Mar\'{i}a},
 doi = {10.20868/upm.thesis.3879},
 month = {June},
 school = {Universidad Polit\'{e}cnica de Madrid},
 title = {{NeOn Methodology for Building Ontology Networks: Specification, Scheduling
and Reuse}},
 url = {http://oa.upm.es/3879/2/MARIA_DEL-_CARMEN_SUAREZ_DE_FIGUEROA_BAONZA.pdf},
 year = {2010}
}

@article{koutsiana2024knowledge,
  title={{Knowledge Prompting: How Knowledge Engineers Use Large Language Models}},
  author={Koutsiana, Elisavet and Walker, Johanna and Nwachukwu, Michelle and Mero{\~n}o-Pe{\~n}uela, Albert and Simperl, Elena},
  journal={arXiv preprint arXiv:2408.08878},
  year={2024}
}

@book{dutilh2012formal,
 author = {Dutilh Novaes, Catarina},
 publisher = {Cambridge University Press},
 title = {{Formal languages in logic: A philosophical and cognitive analysis}},
 year = {2012}
}

@book{novaes2020dialogical,
 author = {Dutilh Novaes, Catarina},
 doi = {10.1017/9781108800792},
 publisher = {Cambridge University Press},
 title = {{The dialogical roots of deduction: Historical, cognitive, and philosophical perspectives on reasoning}},
 year = {2020}
}

@book{schreiber2000knowledge,
 author = {Schreiber, August Th and Schreiber, Guus and Akkermans, Hans and Anjewierden, Anjo and Shadbolt, Nigel and de Hoog, Robert and Van de Velde, Walter and Wielinga, Bob},
 publisher = {MIT Press},
 title = {{Knowledge engineering and management: the CommonKADS methodology}},
 year = {2000}
}

@article{studer1998knowledge,
 author = {Studer, Rudi and Benjamins, V Richard and Fensel, Dieter},
 doi = {10.1016/S0169-023X(97)00056-6},
 journal = {Data \& knowledge engineering},
 number = {1-2},
 pages = {161--197},
 publisher = {Elsevier},
 title = {{Knowledge engineering: Principles and methods}},
 volume = {25},
 year = {1998}
}

@misc{claudecode2025,
  title = {{Claude Code Overview}},
  author = {Anthropic},
  year = {2025},
  url = {https://docs.claude.com/en/docs/claude-code/overview},
  note = {URL accessed: 2025-10-14}
}

@article{forsythe1993engineering,
  title={{Engineering knowledge: The construction of knowledge in artificial intelligence}},
  author={Forsythe, Diana E},
  journal={Social studies of science},
  volume={23},
  number={3},
  pages={445--477},
  year={1993},
  publisher={Sage Publications}
}

@book{brandom1994making,
  title={{Making it explicit: Reasoning, representing, and discursive commitment}},
  author={Brandom, Robert},
  year={1994},
  publisher={Harvard University Press}
}

@inproceedings{restall2005multiple,
  title={{Multiple conclusions}},
  author={Restall, Greg},
  booktitle={Logic, methodology and philosophy of science: Proceedings of the twelfth international congress},
  pages={189--205},
  year={2005},
  organization={Kings College Publications}
}

@book{hlobil2025reasons,
  title={{Reasons for logic, logic for reasons: Pragmatics, semantics, and conceptual roles}},
  author={Hlobil, Ulf and Brandom, Robert B},
  year={2025},
  publisher={Routledge}
}

@misc{github2025issues,
  author = {{GitHub}},
  title = {About issues},
  year = {2025},
  howpublished = {URL accessed: 2025-02-01},
  url = {https://docs.github.com/en/issues/tracking-your-work-with-issues/about-issues}
}

@misc{allen2025elenchus,
  author = {Allen, Bradley P.},
  title = {{Elenchus: Dialectical Opponent for Research}},
  year = {2025},
  howpublished = {GitHub repository},
  url = {https://github.com/bradleypallen/elenchus}
}

@incollection{dutilh2025proofs,
  title={{Proofs as Dialogues: The Enduring Significance of Lakatos for the Philosophy of Mathematical Practice}},
  author={Dutilh Novaes, Catarina},
  booktitle={Proofs and Research Programmes: Lakatos at 100},
  pages={27--46},
  year={2025},
  publisher={Springer}
}

@article{sellars1953inference,
  title={Inference and meaning},
  author={Sellars, Wilfrid},
  journal={Mind},
  volume={62},
  number={247},
  pages={313--338},
  year={1953},
  publisher={JSTOR}
}

@techreport{lebo2013prov,
  author      = {Timothy Lebo and Satya Sahoo and Deborah McGuinness},
  title       = {{PROV-O: The PROV Ontology}},
  institution = {W3C},
  year        = {2013},
  type        = {W3C Recommendation},
  url         = {https://www.w3.org/TR/2013/REC-prov-o-20130430/},
  note        = {https://www.w3.org/TR/prov-o/}
}

@article{moreau2015rationale,
  title={{The rationale of PROV}},
  author={Moreau, Luc and Groth, Paul and Cheney, James and Lebo, Timothy and Miles, Simon},
  journal={Journal of Web Semantics},
  volume={35},
  pages={235--257},
  year={2015},
  publisher={Elsevier}
}

@inproceedings{clark2001knowledge,
  title={Knowledge entry as the graphical assembly of components},
  author={Clark, Peter and Thompson, John and Barker, Ken and Porter, Bruce and Chaudhri, Vinay and Rodriguez, Andres and Thomere, Jerome and Mishra, Sunil and Gil, Yolanda and Hayes, Pat and others},
  booktitle={{Proceedings of the 1st International Conference on Knowledge Capture}},
  pages={22--29},
  year={2001}
}

@article{boose1985,
  author    = {Boose, John H.},
  title     = {A Knowledge Acquisition Program for Expert Systems Based on Personal Construct Psychology},
  journal   = {International Journal of Man-Machine Studies},
  volume    = {23},
  number    = {5},
  pages     = {495--525},
  year      = {1985}
}

@book{ericsson1984,
  author    = {Ericsson, K. Anders and Simon, Herbert A.},
  title     = {Protocol Analysis: Verbal Reports as Data},
  publisher = {MIT Press},
  year      = {1984}
}

@book{cimiano2006,
  author    = {Cimiano, Philipp},
  title     = {Ontology Learning and Population from Text: Algorithms, Evaluation and Applications},
  publisher = {Springer},
  year      = {2006}
}

@inproceedings{clark1997building,
  title={Building concept representations from reusable components},
  author={Clark, Peter and Porter, Bruce},
  booktitle={Proceedings of the fourteenth national conference on artificial intelligence and ninth conference on Innovative applications of artificial intelligence},
  pages={369--376},
  year={1997}
}

@inproceedings{babaei2023llms40l,
  author    = {Babaei Giglou, Hamed and D'Souza, Jennifer and Auer, S{\"o}ren},
  title     = {{LLMs4OL}: Large Language Models for Ontology Learning},
  booktitle = {Proceedings of the International Semantic Web Conference},
  year      = {2023}
}

@inproceedings{yao2025exploring,
  title={Exploring large language models for knowledge graph completion},
  author={Yao, Liang and Peng, Jiazhen and Mao, Chengsheng and Luo, Yuan},
  booktitle={ICASSP 2025-2025 IEEE International Conference on Acoustics, Speech and Signal Processing (ICASSP)},
  pages={1--5},
  year={2025},
  organization={IEEE}
}

@article{wei2023chatie,
  title={{Chatie: Zero-shot Information Extraction via Chatting with ChatGPT}},
  author={Wei, Xiang and Cui, Xingyu and Cheng, Ning and Wang, Xiaobin and Zhang, Xin and Huang, Shen and Xie, Pengjun and Xu, Jinan and Chen, Yufeng and Zhang, Meishan and others},
  journal={arXiv preprint arXiv:2302.10205},
  year={2023}
}

@article{dung1995,
  author    = {Dung, Phan Minh},
  title     = {On the Acceptability of Arguments and its Fundamental Role in Nonmonotonic Reasoning, Logic Programming and n-Person Games},
  journal   = {Artificial Intelligence},
  volume    = {77},
  number    = {2},
  pages     = {321--357},
  year      = {1995}
}

@article{modgil2014,
  author    = {Modgil, Sanjay and Prakken, Henry},
  title     = {The {ASPIC+} Framework for Structured Argumentation: A Tutorial},
  journal   = {Argument \& Computation},
  volume    = {5},
  number    = {1},
  pages     = {31--62},
  year      = {2014}
}

@article{prakken2006,
  author    = {Prakken, Henry},
  title     = {Formal Systems for Persuasion Dialogue},
  journal   = {The Knowledge Engineering Review},
  volume    = {21},
  number    = {2},
  pages     = {163--188},
  year      = {2006}
}

@book{lorenzen1978,
  author    = {Lorenzen, Paul and Lorenz, Kuno},
  title     = {Dialogische Logik},
  publisher = {Wissenschaftliche Buchgesellschaft},
  year      = {1978}
}

@inproceedings{cayrol2005,
  author    = {Cayrol, Claudette and Lagasquie-Schiex, Marie-Christine},
  title     = {On the Acceptability of Arguments in Bipolar Argumentation Frameworks},
  booktitle = {Proceedings of the 8th European Conference on Symbolic and Quantitative Approaches to Reasoning with Uncertainty},
  pages     = {378--389},
  year      = {2005}
}

@article{governatori2004,
  author    = {Governatori, Guido and Maher, Michael J. and Antoniou, Grigoris and Billington, David},
  title     = {Argumentation Semantics for Defeasible Logic},
  journal   = {Journal of Logic and Computation},
  volume    = {14},
  number    = {5},
  pages     = {675--702},
  year      = {2004}
}

@inproceedings{zhang2024ontochat,
  title={Ontochat: a framework for conversational ontology engineering using language models},
  author={Zhang, Bohui and Carriero, Valentina Anita and Schreiberhuber, Katrin and Tsaneva, Stefani and Gonz{\'a}lez, Luc{\'\i}a S{\'a}nchez and Kim, Jongmo and de Berardinis, Jacopo},
  booktitle={European Semantic Web Conference},
  pages={102--121},
  year={2024},
  organization={Springer}
}

@article{kampars2025llm,
  title={LLM-supported collaborative ontology design for data and knowledge management platforms},
  author={Kampars, Janis and Mosans, Guntis and Jogi, Tushar and Roters, Franz and Vajragupta, Napat},
  journal={Frontiers in big Data},
  volume={8},
  pages={1676477},
  year={2025},
  publisher={Frontiers Media SA}
}

@inproceedings{garijo2024llms,
  title={Llms for ontology engineering: a landscape of tasks and benchmarking challenges},
  author={Garijo, Daniel and Poveda-Villal{\'o}n, Mar{\'\i}a and Amador-Dom{\'\i}nguez, Elvira and Wang, Z and Garc{\'\i}a-Castro, Ra{\'u}l and Corcho, Oscar},
  booktitle={The Semantic Web-ISWC},
  year={2024}
}

@article{irving2018ai,
  title={AI safety via debate},
  author={Irving, Geoffrey and Christiano, Paul and Amodei, Dario},
  journal={arXiv preprint arXiv:1805.00899},
  year={2018}
}

@article{li2024improving,
  title={Improving multi-agent debate with sparse communication topology},
  author={Li, Yunxuan and Du, Yibing and Zhang, Jiageng and Hou, Le and Grabowski, Peter and Li, Yeqing and Ie, Eugene},
  journal={arXiv preprint arXiv:2406.11776},
  year={2024}
}

@inproceedings{liang2024encouraging,
  title={Encouraging divergent thinking in large language models through multi-agent debate},
  author={Liang, Tian and He, Zhiwei and Jiao, Wenxiang and Wang, Xing and Wang, Yan and Wang, Rui and Yang, Yujiu and Shi, Shuming and Tu, Zhaopeng},
  booktitle={Proceedings of the 2024 conference on empirical methods in natural language processing},
  pages={17889--17904},
  year={2024}
}

@article{khan2024debating,
  title={Debating with more persuasive llms leads to more truthful answers},
  author={Khan, Akbir and Hughes, John and Valentine, Dan and Ruis, Laura and Sachan, Kshitij and Radhakrishnan, Ansh and Grefenstette, Edward and Bowman, Samuel R and Rockt{\"a}schel, Tim and Perez, Ethan},
  journal={arXiv preprint arXiv:2402.06782},
  year={2024}
}

@misc{pynmms2025,
   author = {Allen, Bradley P.},
   title = {{pyNMMS}},
   year = {2026},
   howpublished = {{PyPI package}},
   url = {https://pypi.org/project/pyNMMS/}
}

@inproceedings{zhao2024improving,
  title={Improving ontology requirements engineering with Ontochat and participatory prompting},
  author={Zhao, Yihang and Zhang, Bohui and Hu, Xi and Ouyang, Shuyin and Kim, Jongmo and Jain, Nitisha and De Berardinis, Jacopo and Mero{\~n}o-Pe{\~n}uela, Albert and Simperl, Elena},
  booktitle={Proceedings of the AAAI Symposium Series},
  volume={4:1},
  pages={253--257},
  year={2024}
}

@article{ciroku2024revont,
  title={Revont: Reverse engineering of competency questions from knowledge graphs via language models},
  author={Ciroku, Fiorela and de Berardinis, Jacopo and Kim, Jongmo and Mero{\~n}o-Pe{\~n}uela, Albert and Presutti, Valentina and Simperl, Elena},
  journal={Journal of Web Semantics},
  volume={82},
  pages={100822},
  year={2024},
  publisher={Elsevier}
}

@incollection{gruninger1995role,
  title={The role of competency questions in enterprise engineering},
  author={Gr{\"u}ninger, Michael and Fox, Mark S},
  booktitle={Benchmarking—Theory and practice},
  pages={22--31},
  year={1995},
  publisher={Springer}
}

@techreport{BrickleyGuha2014,
  author       = {Dan Brickley and R.V. Guha},
  title        = {{RDF Schema 1.1}},
  institution  = {W3C},
  type         = {{W3C} Recommendation},
  year         = {2014},
  month        = feb,
  day          = {25},
  url          = {https://www.w3.org/TR/rdf11-schema/},
  note         = {URL accessed: 2025-02-26}
}

\end{document}